\documentclass{article}

\usepackage{microtype}
\usepackage{tabu}
\newenvironment{proof}{\paragraph{Proof:}}{\hfill$\square$}
\usepackage{subfig}
\usepackage{graphicx}
\usepackage{amsfonts}
\usepackage{booktabs} % for professional tables
\usepackage{float}
\usepackage[utf8]{inputenc}
\usepackage[english]{babel}

\newtheorem{definition}{Definition}[section]

\newtheorem{proposition}{Proposition}[section]

\usepackage{hyperref}

\usepackage[accepted]{icml2019}

\icmltitlerunning{Towards Fair Deep Clustering With Multi-State Protected Variables}

\begin{document}

\twocolumn[
\icmltitle{Towards Fair Deep Clustering With Multi-State Protected Variables}




\begin{icmlauthorlist}
\icmlauthor{Bokun Wang}{to}
\icmlauthor{Ian Davidson}{to}
\end{icmlauthorlist}

\icmlaffiliation{to}{Department of Computer Science, UC Davis}

\icmlcorrespondingauthor{Bokun Wang}{bbwang@ucdavis.edu}

\icmlkeywords{Machine Learning, ICML}

\vskip 0.3in
]



\printAffiliationsAndNotice{}  % leave blank if no need to mention equal contribution

	\begin{abstract}
	Fair clustering under the disparate impact doctrine requires that population of each protected group should be approximately equal 
	in every cluster. Previous work investigated a difficult-to-scale pre-processing step for $k$-center and $k$-median style algorithms for the special case of this problem when the number of protected groups is two. In this work, we consider a more general and practical setting where there can be many protected groups. To this end, we propose Deep Fair Clustering, which learns a discriminative but fair cluster assignment function. The experimental results on three public datasets with different types of protected attribute show that our approach can steadily improve the degree of fairness while only having minor loss in terms of clustering quality.
\end{abstract}

\section{Introduction}

\textbf{Problem Setting.} The purpose of clustering is often given as trying to find groups of instances that are similar to each other but dis-similar to others. Common clustering methods include the spectral clustering \cite{DBLP:conf/nips/NgJW01, DBLP:journals/sac/Luxburg07} that tries to effectively maximize the edges within a sub-graph~(clusters), the Louvain method \cite{blondel2008fast} that tries to maximize the number of edges per cluster, and k-means style algorithms \cite{DBLP:journals/tit/Lloyd82, DBLP:conf/focs/OstrovskyRSS06} that attempt to minimize the distortion.	

An alternative more pragmatic view of clustering is to find equivalence classes where every object in the cluster is effectively equivalent to every other object. This is often how clustering is used in practice: customers in the same cluster are marketed/retained the same way, people in the same community are candidates for the same friendship invitations and images in the same group are given similar tags. In most clustering settings the view of the means justifying the ends is sufficient. That is, any clustering is justifiable so long as it better improves the end result of maximizing click through rate, maximizing number of friendships and so on.

\textbf{Limits of Existing Work.} However, when the clusters are of people or other entities deserving moral status \cite{warren1997moral} then the issue
of fairness of the clustering must be considered. The standard way of ensuring fairness is to
declare some variables as \textit{protected} and these variables are \textit{not} used by the algorithm when forming the clusters but rather are given to the algorithm to ensure the clustering is fair. To our knowledge there is only one theoretical paper on the topic \cite{chierichetti2017fair} which shows a clever way of pre-processing data set into chunks. This step guarantees that when $k$-center and $k$-median algorithms are applied to the chunks they produce fair clusters. However, this pre-processing step is time consuming and experimental results are limited to one thousand or less sized instances. 
Furthermore, existing work even in the supervised learning explores using only \textbf{Boolean protected variables} such as gender, ethnicity and attempted to balance them so that each cluster/class has	equal (within mathematical limits) of males and females, non-whites and whites etc. The area of  multi-state protected variables for clustering has not been studied to our knowledge but many protected status variables are in fact not binary. Consider the classic eight tenants of protected status: sex, race, age, disability, color, creed, national origin, or religion.  age and race. None are \emph{now} considered binary with each having more than two responses in most forms used to collect data. 

\textbf{Challenges and Notion of Fairoids.} The area of fair clustering raises several challenges. 
%for: i)real-valued problem formulation and ii) algorithms.
%		For the challenges in the problem formulation, consider a Boolean protected variable where each protected value occurs in equal amounts (i.e. the number of males and females are the same),
%	then the balancing of protected variable is easy to formulate as a balancing requirement that $\#\mathrm{Male}$ = $\#\mathrm{Female}$. However, %it is unclear how to define ``fairness'' when there exists multi-valued protected variables. ???Is there really another way???
%	Further challenges are raised if the training of the system is under mini-batch setting where not all instances are given at once. We limit the scope of this paper to considering only one real-valued protected variable. %To evaluate fairness, we utilize histogram and Wasserstein distance to compare histogram of protected attribute values in each cluster and the groundtruth. The groundtruth is the most fair histogram we can get, where the values of protected attribute are evenly distributed in the total range. 
A core challenge is that optimizing some measure of good clustering and a good measure of fairness is a challenging concept as the two are fundamentally different. Then the key challenge is to find a formulation that can optimize both objectives simultaneously. In this work, we explore the potential of incorporating fairness consideration in deep clustering models.  The core idea of our method is to use deep learning to learn an embedding that achieves two aims: i) separating the instances into clusters and ii) making the composition of those clusters uniform with respect to the protected attribute. This is achieved by exploiting the idea that the instances with the \textbf{same protected status} form a protected group and inherently a ``fairoid" (fairness centroid). That is if the protected status is race, there would be a fairoid for each of white, black, hispanic etc. Then we wish the cluster centroids to be \textbf{equi-distant} from each and every fairoid to ensure fairness. This is illustrated in Figure \ref{fig:idea} and for an experiment in Figure \ref{fig: tsne_ours}. There are just two protected status types in this simple example. The deep learner simultaneously learns an embedding for both the cluster centroids and the fairoids  under the requirement that each cluster should be equi-distant to each and every fairoid in the embedded space. In Section \ref{moti} we discuss why this is reasonable for ensuring fairness.

\begin{figure}[t]
	\centering
	\includegraphics[width=0.88\linewidth]{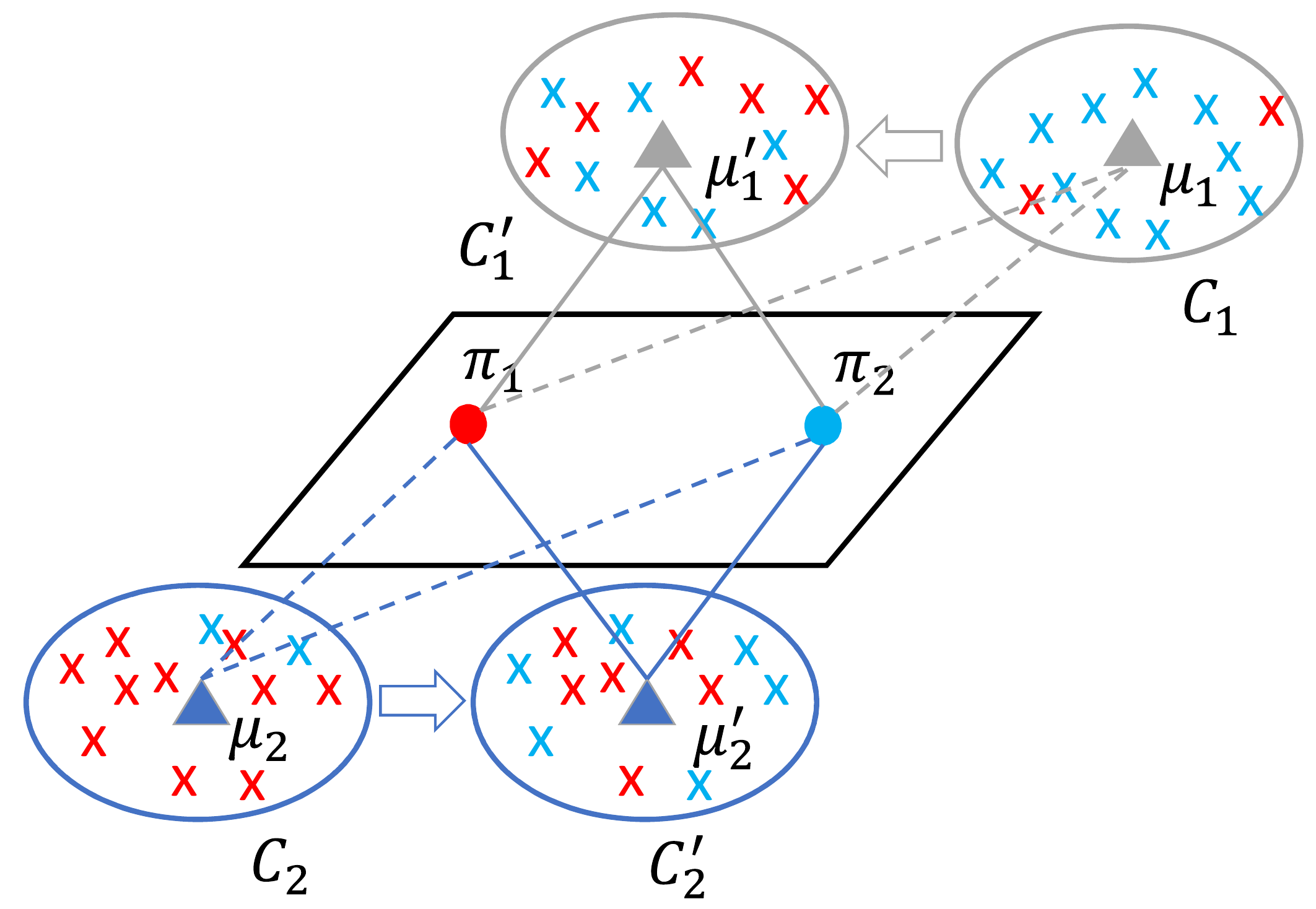}
	\caption{Core idea of our approach. The red and blue circles $\pi_1, \pi_2$ are fairoids of the two protected groups, where colors are coded to match the instances denoted by``$\times$''. $\mu_1, \mu_2$ are cluster centroids before training fairness objective while $\mu_1^{'}, \mu_2^{'}$ are the cluster centroids after optimizing the fairness requirement that the centroids are equi-distant from the fairoids.}
	\label{fig:idea}
\end{figure}

\begin{figure}[t]
	\centering
	\includegraphics[width=0.88\linewidth]{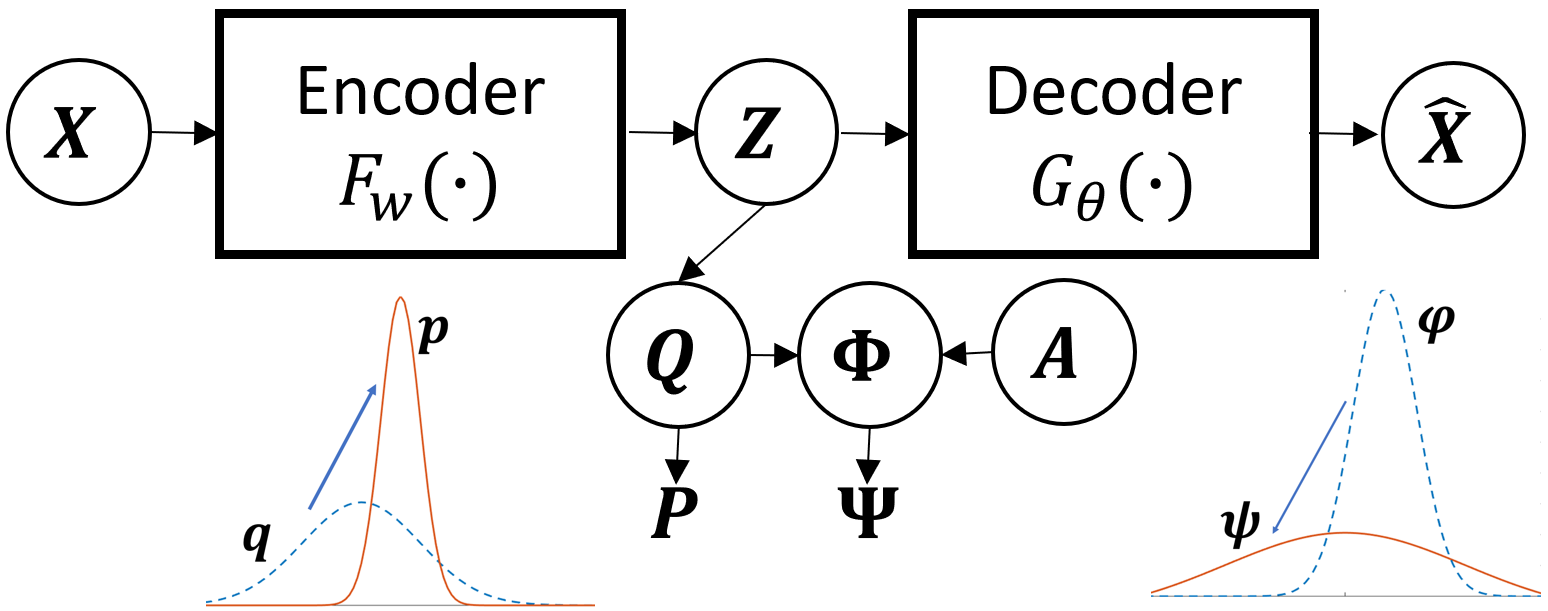}
	\caption{Proposed model for fair deep clustering. $X$ is the original data and encoder $F_{\mathcal{W}}$ maps $X$ into latent representations $Z$. Decoder $G_\theta$ reconstructs $\hat{X}$ from $Z$. $A$ denotes given protected attribute. $Q$ is the soft cluster assignment matrix while $\Phi$ is the ``fairness'' assignment matrix (explained in Section \ref{our_model}). $P$ and $\Psi$ are self-training targets, for sharpening and smoothing purposes.}
	\label{fig:architecture}
\end{figure}

The contributions of this paper are summarized as follows:
\begin{itemize}
	\setlength{\itemsep}{2pt}
	\setlength{\parsep}{2pt}
	\setlength{\parskip}{2pt}
	\item We define fairness problem in clustering for multi-state protected attributes.This has significant practical importance as most
	protected variables are inherently multi-state not binary.
	\item We propose a novel mechanism towards finding fair clusters. All instances for a particular protected variable value defines a fairoid and finding centroids equi-distant to the fairoids produces fair clustering.
	\item We try out method in a deep learning setting and demonstrate its effectiveness against a number of classic baselines.
\end{itemize}

\section{Related Work} 
\subsection{Deep Clustering}

Recent deep clustering work could be broadly categorized into \textit{two-phase} and \textit{one-phase} approaches. The two-phase approaches \cite{DBLP:conf/icml/LawUZ17, DBLP:journals/corr/abs-1801-01587} learn a parameterized mapping function $F_\mathcal{W}: \mathbb{R}^D \rightarrow \mathbb{R}^d$  ($D \gg d$) from original data or high-dimensional feature space to a latent space that is good for classical clustering methods like $k$-means and then apply $k$-means on that space. Alternatively, one-phase approaches \cite{DBLP:conf/icml/XieGF16, DBLP:conf/cvpr/YangPB16} directly learn an inductive assignment function to allocate data points into $K$ clusters. A representative one-phase work is deep embedded clustering \cite{DBLP:conf/icml/XieGF16}, which trains the network parameters based on self-training strategy. This paradigm has been adopted by others \cite{DBLP:conf/icml/YangFSH17, dizaji2017deep}. In this work, we focus on the fairness issue in one-stage deep clustering, which to our knowledge has not been investigated before.

\subsection{Fairness in Machine Learning} 

Recently, inherent unfairness\footnote{http://fairml.how/tutorial/index.html/} and algorithmic bias\footnote{http://approximatelycorrect.com/2016/11/07/the-foundations-of-algorithmic-bias/} in machine learning algorithms have attracted much research interest. There are two main doctrines regarding fairness in machine learning: \textit{Group} Fairness and \textit{Individual} Fairness. The group notion, also called statistical parity, requires parity of some specific statistical measure across all protected demographic groups. Considering a automatic college admission system, the admission \emph{probabilities} of Caucasian, Black, Asian, etc. students should be the same. On the other hand, the individual notion requires pair of individuals with similar protected information should also be treated similarly with respect to a task-specific metric. For example, people with similar heath characteristics should receive similar treatments \cite{DBLP:conf/innovations/DworkHPRZ12}. 
%???Because each person's health characteristic is unique so that we cannot separate people into several protected groups.??? 
The two main notions of fairness from social science and the law are \textit{disparate impact} and \textit{disparate treatment}. Disparate treatment means a decision on individual changes when the protected attribute is changed. For example, a job applicants is given a job offer but it later retracted when their gender is revealed. Disparate impact denotes that the decision outcome disproportionately benefits or hurts users belonging to different protected groups. Fairness in this work always refers to \emph{group fairness and disparate impact} notions as others have \cite{chierichetti2017fair}, in particular the parity we require is given by Equation \ref{fair_def_T}. 

%	Apart from flourishing developments in fair supervised learning, it has been shown that unfairness issue also exists in unsupervised learning field, like clustering \cite{chierichetti2017fair} and unsupervised dimensionality reduction \cite{DBLP:journals/corr/abs-1811-00103}. 
%Taking gender as an example, learning fair dimensionality reduction means the average of reconstruction error of male data points should be equal to that of female.  

Our work is related to supervised fair representation learning \cite{DBLP:conf/icml/MadrasCPZ18, DBLP:conf/icml/ZemelWSPD13}, which learns fairer latent representation but only for classification or transfer learning tasks hence is not directly comparable. Our work instead directly imposes fairness requirements on cluster assignment function and centroids. The work most related to ours is \cite{chierichetti2017fair}. They study the fair clustering problem when protected attribute is binary (\textit{e.g.}, male/female). However, it is worth noting that a lot of protected attributes could not be represented as binary values, like age, income, etc. Besides, there might be multiple binary protected attributes required to be balanced simultaneously (\textit{e.g.}, in each cluster $\#\mathrm{Male} = \#\mathrm{Female}$ and $ \#\mathrm{Single} = \#\mathrm{Married}$). In fact, both single and multiple \textit{binary} protected variables are just the special cases of single multi-state protected variable.
%???Is this true. I think you mean to say that if you allow more than 1 in n encoding your method could be used for multiple binary protected variables???.  

Our paper tries to attempt settings not addressed in earlier work \cite{chierichetti2017fair} namely: i) multiple protected variables and ii) scalable fair clustering. We attempt the  fairness in clustering problem for multi-state protected attribute which has $T \geq 2$ unique values.  Moreover, the approach in \cite{chierichetti2017fair} is not scalable as it utilizes $\mathcal{O}(N^3)$ combinatorial algorithm so that it cannot been appied if $N$ is large. To this end, we propose a scalable algorithm via mini-batch training. Our approach jointly optimize clustering objective from DEC \cite{DBLP:conf/icml/XieGF16} and our novel fairness objective. Potentially, Our fairness objective could also be combined with better topology-preserving clustering objectives like those in \cite{DBLP:journals/corr/abs-1801-01587, DBLP:journals/corr/abs-1803-01449} but clustering accuracy is not our major concern.

\section{Definition and Preliminary}
Let $X \in \mathbb{R}^{N \times D}$ denote $N$ data points with $D$-dimension features. The aim of clustering $\mathcal{C}$ is to learn an assignment function $\alpha: \mathbf{x} \rightarrow \{1, ..., K\}$, which allocates $N$ data points into $K$ disjoint clusters $C^1, C^2, ..., C^K$ ( $\bigcap_{k=1}^{K}C^k = \emptyset$). In addition let the protected attribute annotations $A\in \mathbb{R}^{N}$ be available for fairness consideration. If there are $T$ unique values in $A$, the data points $X$ can be naturally partitioned into $T$ disjoint protected groups $\mathcal{G}^1, \mathcal{G}^2, ..., \mathcal{G}^T$ ($\bigcap_{t=1}^{T}\mathcal{G}^t = \emptyset$) by another function $\chi: \mathbf{x} \rightarrow \mathcal{A}$. Here $\mathcal{A}$ is a finite set of unique values in protected attribute, $\mathcal{A} = \{a_1, a_2, ..., a_T\}$. When considering fairness in the clustering problem, the aim is to get a fair partition \textit{w.r.t.} specific protected attribute $A$ but a useful clustering \textit{w.r.t.} topology of original data.

\textbf{Previous Work With Binary Protected Variables.} Previous work \cite{chierichetti2017fair} discussed the definition of fair clustering under the special case $T = 2$. The ``perfectly balanced'' cluster $C^k$ can be defined as follows, which is colloquially that all protected states are equally likely. 
{\setlength\abovedisplayskip{5pt}
	\begin{equation}
	\label{fair_def_2}
	P(\mathbf{x} \in \mathcal{G}^1 \big| \mathbf{x} \in C^k) = P(\mathbf{x} \in \mathcal{G}^{2} \big| \mathbf{x} \in C^k)
	\end{equation}
	\setlength\belowdisplayskip{5pt}}
They then define the fairness measure called ``balance'' inspired by $p\%$-rule \cite{Biddle}. Here $P(x \in \mathcal{G}^t \big| \mathbf{x} \in C^K) = \frac{N_{k}^t}{|C^k|}, t=1, 2$ and $N_{k}^t$ is the cardinality of the set $\{i|\mathbf{x}_i \in C^k \land \mathbf{x}_i \in \mathcal{G}^t\}$. 	
\begin{equation}
\label{p-rule}
\mathrm{balance}(\mathcal{C}^k) = \min\Big(\frac{N_k^1}{N_k^2}, \frac{N_k^2}{N_k^1}\Big) \in [0, 1]
\end{equation} 

\textbf{Generalization to Multi-State Protected Variables.} 	If we  consider a more general setting, that is the protected attribute is not necessarily binary, we can generalize the definition in Formula \ref{fair_def_2} into the following: 

\begin{definition}
	\label{def_T}
	\textbf{Condition of Fairest cluster}. According to the notions of group fairness and disparate impact the $k^{th}$ cluster $C^k$ is fair if the chance of encountering any protected status is equal. Formally, $\forall t_1, t_2 \in 1, 2, ..., T$:
	%$\forall t_1, t_2 \in \{1, ..., T\}$ , the probability of $\mathbf{x} \in \mathcal{C}^k$ belongs to protected group $\mathcal{G}^{t_1}$ equals to the probability of that $\mathbf{x}$ belongs to protected group $\mathcal{G}^{t_2}$:
\end{definition}
\begin{equation}
\label{fair_def_T}
P(\mathbf{x}\!\in\!\mathcal{G}^{t_1}\big|\mathbf{x}\!\in\!C^k)\!=\! P(\mathbf{x}\!\in\!\mathcal{G}^{t_2}\big|\mathbf{x} \in C^k)
\end{equation}

\begin{proposition}
	\label{theorem_1}
	Definition \ref{def_T} is equivalent to the $k^{th}$ cluster's  histogram of the protected attribute $h^k$ having a uniform distribution. That is  $h^k \sim \mathcal{U}\big(\inf{\mathcal{A}}, \sup{\mathcal{A}}\big)$.
\end{proposition}
\begin{proof}
	Value $h^k_t$ of the histogram $h^k$ on the bin $a_t$ for $k^{th}$ cluster is $h^k_t = \frac{|\mathcal{I}^k_t|}{|C^k|} = \frac{\sum_i P(\mathbf{x}_i \in C^k \land \mathbf{x}_i \in \mathcal{G}^t)}{\sum_i P(\mathbf{x}_i \in C^k)}$, where $\mathcal{I}^k_t = \{i\big|\chi(\mathbf{x}_i) = a_t \land \alpha(\mathbf{x}_i) = k\}$. If $P(\mathbf{x}_i \in \mathcal{G}^{t_1} \big| \mathbf{x}_i \in C^k) = P(\mathbf{x}_i\!\in \mathcal{G}^{t_2} \big| \mathbf{x}_i \in C^k)$ and $p(\mathbf{x}_i\!\in\! \mathcal{G}^t), p(\mathbf{x}_i \in C^k)$ are determined from clustering $\mathcal{C}$ and protected attribute annotation $A$, then $h^{k}_{t_1} = h^{k}_{t_2}, \forall t_1, t_2 \in 1, ..., T$. This leads to $h^{k}_{t} = \frac{1}{T}$, concluding the proof.
	\label{fair_def_T_eq}
\end{proof}

In the $k^{th}$ cluster ${C}^k$, there exists $T$ subgroups $G_k^1 \subset \mathcal{G}^1, G_k^2\subset \mathcal{G}^2, ..., G_k^T\subset \mathcal{G}^T$ ($\forall t,~|G_k^t| \geq 0$) according to protected attribute annotations. One challenging problem is how to measure disparity among those subgroups $G_k^1, G_k^2, ..., G_k^T$. If $T = 2$, we can simply compare the ratios $|G_k^1|/|C^k|$ and $|G_k^2|/|C^k|$ through either $p\%$-rule or Calders-Verwer score \cite{DBLP:journals/datamine/CaldersV10}. The former way is utilized in \cite{chierichetti2017fair} to define balance score in Formula \ref{p-rule} while the latter way could be written as $\Big|P(\mathbf{x} \in \mathcal{G}^1\big|\mathbf{x} \in C^k) - P(\mathbf{x} \in \mathcal{G}^2\big|\mathbf{x} \in C^k)\Big|$. 

When protected attribute $A$ is not binary, how to measure disparity among those subgroups $G_k^1, ..., G_k^T$ becomes more challenging. 
Based on Proposition \ref{theorem_1} above, we define fairness measure of $\mathcal{C}^k$ as the discrepancy between real distribution $\mathcal{H}_k$ that $h_k \sim \mathcal{H}_k$ and uniform distribution $\mathcal{U}\big(\inf{\mathcal{A}}, \sup{\mathcal{A}}\big)$. 

\begin{definition}
	\label{measure_T}
	\textbf{Earth-mover Distance as a Fairness Measure}. Let the empirical histogram for the $k^{th}$ cluster be $h_k \sim \mathcal{H}_k$ and the optimal histogram to ensure fairness be $u_k \sim \mathcal{U}\big(\inf{\mathcal{A}}, \sup{\mathcal{A}}\big)$. Then a fairness measure is simply the distance between them.
	We utilize the Wasserstein distance $W_p$ as $\mathcal{D}_k(\cdot, \cdot)$:
	\begin{equation}
	\mathcal{D}_k(\mathcal{H}_k, \mathcal{U}) = \Big(\inf_{J_k \in \mathcal{J}_k}\int \|h_k - u_k\|^p dJ_k\Big)^{1/p} 
	\end{equation}
	Here $\mathcal{J}_k(\mathcal{H}_k, \mathcal{U})$ is the joint distribution of $\mathcal{H}_k$ and $\mathcal{U}$. 
\end{definition}

\begin{proposition}
	\label{cv-score} \textbf{Equivalence of Our Approach For Two Stated Protected Variables.}
	When protected attribute is binary ($T=2$) and Wasserstein distance with $p=1$ is used, Definition \ref{measure_T} degenerates into Calders-Verwer score.
\end{proposition}
\begin{proof}
	The Calders-Verwer score of cluster $k$ could be written as $\Big|P(\mathbf{x}\in \mathcal{G}^1\big|\mathbf{x}\in C^k) - P(\mathbf{x}\in \mathcal{G}^2\big|\mathbf{x}\in C^k)\Big| = C_1|h_k^1 - h_k^2|$ ($C_1$ is a constant and $h_k^1, h_k^2$ denote values on two bins). $\mathcal{D}_k(\mathcal{H}_k, \mathcal{U}) = C_2\sum_{t=1}^2\big|h_k^t - \frac{1}{2}\big|$. Considering $h_k^1h_k^2\geq 0$ and $h_k^1 + h_k^2 =1$, our fairness measure in Definition. \ref{measure_T} and Calders-Verwer score are equivalent if $p=1$ and $T=2$, concluding the proof. 
\end{proof}

As shown in previous work \cite{DBLP:conf/aistats/ZafarVGG17}, fairness results measured by Calders-Verwer score \cite{DBLP:journals/datamine/CaldersV10} and $p\%$-rule show similar trend. Based on Proposition \ref{cv-score}, we expect that when protected attribute is binary, fairness measured by our protocol should also show similar trend as ``balance'' defined in \cite{chierichetti2017fair}, which is based on $p\%$-rule.

%\begin{figure*}[h]
%	\centering
%	\begin{tabular}{@{}cccc@{}}
%		\includegraphics[width=.25\textwidth]{./pics/kmeans_1} &
%		\includegraphics[width=.25\textwidth]{./pics/kmeans_2} &
%		\includegraphics[width=.25\textwidth]{./pics/sc_1} &
%		\includegraphics[width=.25\textwidth]{./pics/sc_2}  
%	\end{tabular}
%	\caption{Normalized age frequency histograms of two clusters from k-means and spectral clustering.}
%	\label{toy}
%\end{figure*}

\section{Deep Learning Formulation}

In this section we outline our deep clustering formulation for fairness. We first begin by explaining why finding centroids that are equi-distant from the fairoids is desirable.	

\subsection{Overview and Use of Fairoids}\label{moti}
In our deep clustering formulation, a latent representation $\mathbf{z} = F_\mathcal{W}(\mathbf{x})$ and cluster assignment function $\alpha: \mathbf{z} \rightarrow \{1, ..., K\}$ are simultaneously learned. Thus the $K$ centroids $\mu_1, \mu_2, ..., \mu_K$ of clusters $C^1, C^2, ..., C^K$ are simply $\mu_k = \mathbb{E}_{\alpha(\mathbf{z}_i) = k}[\mathbf{z}] = \frac{1}{|C^k|}\sum_{i=1}^{N} \mathbf{z}_i 1[\alpha(\mathbf{z}_i) = k]$. Each fairoids $\pi_1, \pi_2, ..., \pi_T$ of protected groups $\mathcal{G}^1, \mathcal{G}^2, ..., \mathcal{G}^T$ is represented as $\pi_t = \mathbb{E}_{\chi(i)=a_t}[\mathbf{z}] = \frac{1}{|\mathcal{G}^t|}\sum_{i=1}^{N} \mathbf{z}_i 1[\chi(i) = a_t]$, given the protected attribute annotations $A$. 

Consider the fairness requirement in Definition. \ref{def_T}, which requires that in cluster $C^k$,  $P(\mathbf{x}\in\mathcal{G}^{t_1}\big|\mathbf{x}\in C^k) = P(\mathbf{x}\in\mathcal{G}^{t_2}\big|\mathbf{x}\in C^k), \forall t_1, t_2 = 1, ..., T$. When performing clustering on latent representations, we can re-write it into $P(\chi(i)=t_1\big|\alpha(\mathbf{z}_i)=k)\!=\!P(\chi(i)=t_2\big|\alpha(\mathbf{z}_i)=k)$ where $\chi$ just returns the protected state of an instance. In deep clustering, the requirement that each cluster should have approximately equal number of each protected status individuals can be achieved by making each cluster centroid to be equi-distant to each and every fairoid. Suppose we have a set instances belonging to the $k^{th}$ cluster $X_{k} = \{\mathbf{x}_i\big|\mathbf{x}_i \in C^k\}$ and set of instances belonging to $t^{th}$ protected group  $X_{t} = \{\mathbf{x}_i\big|\mathbf{x}_i \in \mathcal{G}^t\}$. Both $X_t$ and $X_k$ are encoded into latent representations $Z_t$, $Z_k$ via the encoder $F_{\mathcal{W}}: \mathcal{X} \rightarrow \mathcal{Z}$. Larger overlap between instances of $C^k$ and $\mathcal{G}^t$ means the $k^{th}$ cluster is more monochromatic towards protected group $\mathcal{G}^t$. The overlap could be computed by Maximum Mean
Discrepancy (MMD) \cite{DBLP:conf/nips/GrettonBRSS06} between $Z_k \sim U_k$ and $Z_t \sim V_t$:  
{\setlength\abovedisplayskip{5pt}
	\begin{equation}
	\label{mmd}
	\mathrm{MMD}(U_k, V_t) = \Big\|\mathbb{E}_{U_k}\big[F_\mathcal{W}(X_k)\big] - \mathbb{E}_{V_t}\big[F_\mathcal{W}(X_t)\big]\Big\|_{\mathcal{Z}}
	\end{equation}
	\setlength\belowdisplayskip{5pt}
	It can be seen that $Z_k\!\!=\!\! F_\mathcal{W}(X_k)$, $Z_t\!\!=\!\! F_\mathcal{W}(X_t)$ and  $\mathbb{E}_{U_k}\big[Z_k\big]\!=\!\mu_k$,  $\mathbb{E}_{V_t}\big[Z_t\big]\!=\!\pi_t$. Thus the fairness requirement of $k^{th}$ cluster in equation \ref{fair_def_T} can be transformed into: 
	\begin{equation}
	\label{mmd_eq}
	\mathrm{MMD}(U_{k}, V_{t_1}) = \mathrm{MMD}(U_{k}, V_{t_2}), \forall t_1, t_2 \in 1, ..., T
	\end{equation}
	
	\subsection{A Deep Fair Clustering Model} 
	\label{our_model}
	Based on the motivation above, we propose an inductive model (Figure \ref{fig:architecture}) to directly learn the mapping function $\alpha$ that transforms $X$ into soft assignment $Q \in \mathbb{R}^{N \times K}$, which is fairer \textit{w.r.t.} protected attribute $A \in \mathbb{R}^N$ but still discriminative enough. At the beginning, cluster centroids $M = \{\mu_k\}_{k=1}^K$ and fairoids $\Pi = \{\pi_t\}_{t=1}^T$ can be calculated from input data $X$, given protected attributes $A$, as well as the initial encoding function $f_\mathcal{W}$. During the training process, clustering objective and fairness objective are jointly optimized until convergence to refine the network parameter $\mathcal{W}$, cluster centroids $M$, and fairoids  $\Pi$ for improving clustering performance as well as fairness. 
	
	%\subsubsection{Soft Assignments}\label{soft}
	
	As described in motivation part, we need to model probability $P(\alpha(\mathbf{z})=k)$ for clustering objective and conditional probability $P(\chi(\mathbf{z})=t \big| \alpha(\mathbf{z})=k)$ for fairness objective. By using student's t-distribution as kernel, the probabilities above are realized by ``soft assignments'': Following \cite{DBLP:conf/icml/XieGF16}, $q_{ik} = P(\alpha(\mathbf{z}_i)=k)$ is written as: 
	{\setlength\abovedisplayskip{5pt}
		\begin{equation}
		q_{ik} = \frac{(1+\|\mathbf{z}_i - \mu_k\|^2/\alpha)^{-\frac{\alpha+1}{2}}}{\sum_{k'}(1+\|\mathbf{z}_i - \mu_{k'}\|^2/\alpha)^{-\frac{\alpha+1}{2}}}
		\end{equation}
		\setlength\belowdisplayskip{5pt}} 
	In this work, we also model the conditional probability $\phi_{kt} = P\big(\chi(i) = t\big| \alpha(\mathbf{z}_i) = k\big)$ based on MMD in Formula \ref{mmd} and kernelized by student's t-distribution: 
	{\setlength\abovedisplayskip{5pt}
		\begin{equation}
		\phi_{kt} = \frac{(1+\|\mu_{\alpha(\mathbf{z}_i)} - \pi_t\|^2/\alpha)^{-\frac{\alpha+1}{2}}}{\sum_{t'}(1+\|\mu_{\alpha(\mathbf{z}_i)} - \pi_{t'}\|^2/\alpha)^{-\frac{\alpha+1}{2}}}
		\end{equation} 
		\setlength\belowdisplayskip{5pt}} 	
	Thus, the requirement in Definition. \ref{def_T} could be expressed as $\phi_{kt_1} = \phi_{kt_2}, \forall t_1, t_2 \in 1, ..., T$. This could be understood as: for any $t_1, t_2 \in 1, ..., T$, the maximum mean discrepancy between (distributions of) latent representations $Z_k$ of instances in cluster $C^k$ and $Z_{t_1}$ of instances in protected group $\mathcal{G}^{t_1}$ should be equal to that between $Z_k$ and latent representations $Z^{t_2}$ of instances in another protected group $\mathcal{G}^{t_2}$. The basic idea is illustrated in Figure \ref{fig:idea}.   
	
	In our work, both of clustering objective and fairness objective are trained by minimizing the KL divergence between soft assignments $Q, \Phi$ and their corresponding self-training targets $P, \Psi$. The overall objective $\mathcal{L}$ combines $\mathcal{L}_{fr}$ with the clustering loss $\mathcal{L}_{cl}$ defined in \cite{DBLP:conf/icml/XieGF16}: 
	\begin{eqnarray}
	\label{overall_loss}
	& \mathcal{L}_{fr} = \mathrm{KL}(\Psi||\Phi) = \sum_k\sum_t \psi_{kt}\log\frac{\psi_{kt}}{\phi_{kt}}\\
	& \mathcal{L} = \mathcal{L}_{cl} + \gamma \mathcal{L}_{fr} = \mathrm{KL}(P||Q) +\gamma \mathrm{KL}(\Psi||\Phi)
	\end{eqnarray}
	
	The goal of self-training is to impose various constraints on soft assignments $Q, \Phi$ for various expected properties.
	
	\subsection{Self-training: Sharpening or Smoothing?}
	
	The crucial point is how to choose self-training targets $P, \Psi$ for soft assignments $Q, \Phi$, which depends on the expected properties of $Q, \Phi$. Interestingly, the expected properties of fairness target distribution $\Psi$ and clustering target distribution $P$ are at opposite poles.
	
	\subsubsection{One-hot constraint on $Q$: Sharpening}
	The intuition of target $P$ has been demonstrated in \cite{DBLP:conf/icml/XieGF16}, which is imposing \textit{one-hot constraint} on each row of $Q$. The constraint requires assignment function to choose exactly one cluster out of $K$ and put more emphasis on  assignment with higher confidence to improve the cluster purity. This could be done by set targets $P$ as normalized $Q^2$. Then optimizing the clustering objective $\mathcal{L}_{cl}$ is ``sharpening'' the soft assignment $Q$. 
	
	\subsubsection{Fairness constraint on $\Phi$: Smoothing}
	
	On the other hand, recalling the fairness constraint $\phi_{kt_1} = \phi_{kt_2}, \forall t_1, t_2 \in 1, ..., T$, we require the kernelized distances $\phi_k$ between cluster centroids $\mu_k$ and fairoids $\pi_1, \pi_2, ..., \pi_T$ are the same. Thus, the row distributions of $\Phi$ should be ``smoothed''. Ideally, it should be smoothed into uniform distribution. Based on this constraint, we define the self-training target $\Psi$ as: 
	{\setlength\abovedisplayskip{5pt}
		\begin{equation}
		\label{smoothed_target}
		\psi_{kt} = \frac{\hat{\phi}_{kt}/f_t}{\sum_t' \hat{\phi}_{kt'}/f_t'}
		\end{equation}
		\setlength\belowdisplayskip{5pt}} 
	Here $\hat{\phi}_{kt} = \sqrt[\beta]{\phi_{kt}+\epsilon}$, $\beta\geq 2$ and $\epsilon$ is a small number for numerical stability. $f_t = \sum_t \phi_{kt} $ is the frequency over histogram bins $a_1, a_2, ... a_T$. The straightforward visualization of how $P$ and $\Psi$ work could be found in Figure \ref{fig:architecture}. 
	
	\subsection{Mini-batch Training}
	
	We also use stacked denoising autoencoder (SDAE) as base model like previous work \cite{DBLP:conf/icml/XieGF16, DBLP:conf/icml/YangFSH17}. The initial cluster centroids $M\in \mathbb{R}^{N \times d}$ are from $k$-means++ on initial latent representation $Z$. After that, we simultaneously optimize the clustering objective and fairness objective via stochastic gradient descent (SGD). 
	
	To enable minibatch training, we need to compute empirical $\Phi$ based on corresponding centroids $\mu_{\alpha(\mathbf{z}_i)} \in \mathbb{R}^{K \times d}$ in each size-$n$ minibatch. We calculate $M = \{\mu_{\alpha(\mathbf{z}_i)}\}_{i=1}^{n}$ from assignment $Q \in \mathbb{R}^{n \times K}$ and its target $P \in \mathbb{R}^{n \times K}$, as well as the latent representation $Z_j \in \mathbb{R}^{n \times d}$ in the $j^{th}$ minibatch \cite{DBLP:journals/corr/kmeans_NMF}:
	\begin{eqnarray}
	\label{approx}
	& M_j \approx (P_j^TP_j)^{-1}{P}_j Z_j 
	\end{eqnarray} 
	
	\subsection{Implementation Details}
	
	Regarding network architecture, we follows the design of previous deep autoencoder-based methods with the following dimensions $D$-500-500-2000-$d$-2000-500-500-$D$.The learning rates for pre-training and training stages are 0.1 and 0.01, respectively. During layer-wise pretraining, dropout rate 0.2 is used. All experiments are performed with minibatch size 256. On HAR dataset, we pretrain the stacked denoising autoencoder in greedy layer-wise manner for 150 epochs and global manner for 100 epochs (150 and 75 for Adult dataset, 100 and 50 for D\&S dataset). We set the dimension of latent representations $d$ for three datasets as the number of clusters. There are two hyper-parameters in our work $\gamma$ in Formula. \ref{overall_loss} and $\beta$ in Formula. \ref{smoothed_target}. In this work we fix $\beta = 1000$ (although other values of $\beta$ might improve performance). $\gamma$ is the only tunable hyper-parameter in this work. Tuning $\gamma$ can control the trade-off between two objectives. In Section. \ref{tradeoff_disc} we discussed its effects.
	
	Before minibatch training, we store $M$, $Q$ and its target $P$ from $k$-means++ initialization and pre-compute $\Pi$, $\Phi$ and its target $\Psi$ according to Formula. \ref{approx} and given protected attribute annotations $A$. During each training epoch, the self-training targets $P_j$ and $\Psi_j$ in $j^{th}$ minibatch are directly sampled from $P$ and $\Psi$. The overall objective function is optimized with respect to network parameters $\mathcal{W}, \theta$ and cluster centroids $M$. Our implementation is based on PyTorch library \cite{paszke2017automatic}. All  experimental results are averaged performance of \textbf{10 runs} with random restarts. 
	
	\section{Experiments}
	
	\subsection{Experimental Settings}
	
	\subsubsection{Datasets}
	
	We evaluate the performance of our proposed model on three datasets with protected attributes. The statistics of these three datasets are shown in Table. \ref{stats}. 
	\begin{itemize}
		\setlength{\itemsep}{2pt}
		\setlength{\parsep}{2pt}
		\setlength{\parskip}{2pt}
		\item \textbf{Human Action Recognition (HAR)} \cite{DBLP:conf/esann/AnguitaGOPR13}. This dataset contains 10,299 instances in total with captured action features for 30 participants. There are 6 actions in total which serves are groundtruth for clustering. Identity of 30 persons is used as protected attribute, which is multi-state.
		\item \textbf{Adult Income dataset (Adult)}\footnote{https://archive.ics.uci.edu/ml/datasets/adult}. It contains 48,842 instances of information describing adults from the 1994 US Census. The groundtruth for clustering is annual income (greater or less than 50). Gender information (Male/Female) is used as protected attribute.
		\item \textbf{Daily and Sports Activity (D\&S)} \cite{DBLP:journals/pr/AltunBT10}. This dataset has 9,120 sensor records of human daily and sports activities. The categorical information such as ``Moving on Stairs'', ``Playing basketball'', ``Rowing'', etc is used as groundtruth for clustering. 8 names of participants are used as protected attribute.
	\end{itemize} 
	We formulate our train/test data sets by randomly sampling according to the splits in Table. \ref{stats}. We evaluate all of the approaches on the test sets which are unseen during training. For spectral clustering approach SEC \cite{DBLP:conf/ijcai/NieXTZ09}, computing Laplacian matrices for the whole training sets are infeasible. We downsample those dataset into 2000 data points. Out-of-sample extension is done by allocating each test data point to its nearest centroid.
	
	\begin{table}[h]
		\begin{center}
			\setlength{\tabcolsep}{1.2mm}{
				\begin{tabular}{cccc}
					\toprule[2pt]
					Datasets & Train/Test & Protected Attribute &$\#\mathrm{cluster}$\\
					\hline HAR &   9,099/1,200 & Person Identity (30) & 6\\
					\hline Adult &   30,165/15,060 & Gender (2) & 2\\
					\hline D\&S &   8,320/800 & Names (8) & 10\\
					\bottomrule[2pt]
				\end{tabular}
			}
		\end{center}
		\caption{Dataset Statistics. In the third column, $(\cdot)$ denotes number of unique states in protected attributes.}
		\label{stats}
	\end{table}
	
	\subsubsection{Baselines}
	
	We compare our method, in terms of both clustering quality and fairness, with our base model DEC \cite{DBLP:conf/icml/XieGF16} and other representative deep or shallow clustering approaches. Shallow clustering baselines including centroid-based $k$-means++ \cite{DBLP:conf/focs/OstrovskyRSS06}, density-based DBSCAN \cite{DBLP:conf/kdd/EsterKSX96}, as well as spectral clustering based SEC \cite{DBLP:conf/ijcai/NieXTZ09}. Apart from DEC, We also compare with other strong baselines that
	use deep autoencoders including Autoencoder + $k$-means (AE+KM) and Deep Cluster Network (DCN) \cite{DBLP:conf/icml/YangFSH17}. We do not compare with previous work \cite{chierichetti2017fair} because: (1) it does not work for datasets with multi-state protected attribute; (2) the complexity of their algorithm is $\mathcal{O}(N^3)$ which is not tractable to our datasets; (3) and their codes are not released. 
	
	For DBSCAN and $k$-means++, we use their implementations of scikit-learn library \cite{scikit-learn}. For DCN, SEC and DEC, we utilize the authors' open-sourced codes. All hyper-parameters in those algorithms are carefully tuned according to instructions in those papers.
	
	\subsubsection{Measures}
	To evaluate the clustering quality of our approach and compared methods mentioned above, we use the widely used clustering accuracy (ACC) and normalized mutual information (NMI) \cite{DBLP:journals/jmlr/StrehlG02} metrics. Higher values for both of these metrics indicate better performance. 
	
	To evaluate fairness of clustering, we use the protocol defined in Definition. \ref{measure_T}. We use Wasserstein distance as the distance metric $\mathcal{D}$ named ``FWD'', where smaller value denotes better performance. For Adult dataset whose protected attribute is binary, we also evaluate the methods in terms of balance score as described in Equation \ref{p-rule} from previous work \cite{chierichetti2017fair}. 
	
	\subsection{Results and analysis}
	
	In this section, we present our experimental results that answer the following research questions: 
	
	\begin{itemize}
		\setlength{\itemsep}{2pt}
		\setlength{\parsep}{2pt}
		\setlength{\parskip}{2pt}
		\item Q1: How fair are existing representative clustering approaches (both deep and non-deep)?
		\item Q2: Can our proposed model improve accuracy and fairness over its base model DEC and other representative approaches? When the number of clusters $K$ is increasing, can our approach steadily learn fairer clusters than baseline methods?
		\item Q3: When the protected attribute is binary, Does our fairness measure show similar trends to the balance metric defined in \cite{chierichetti2017fair}?
		\item Q4: When changing hyper-parameter $\gamma$ of our model, what is the trade-off between clustering accuracy and fairness?
	\end{itemize}
	\begin{figure}[h]
		\centering
		\includegraphics[width=0.88\linewidth]{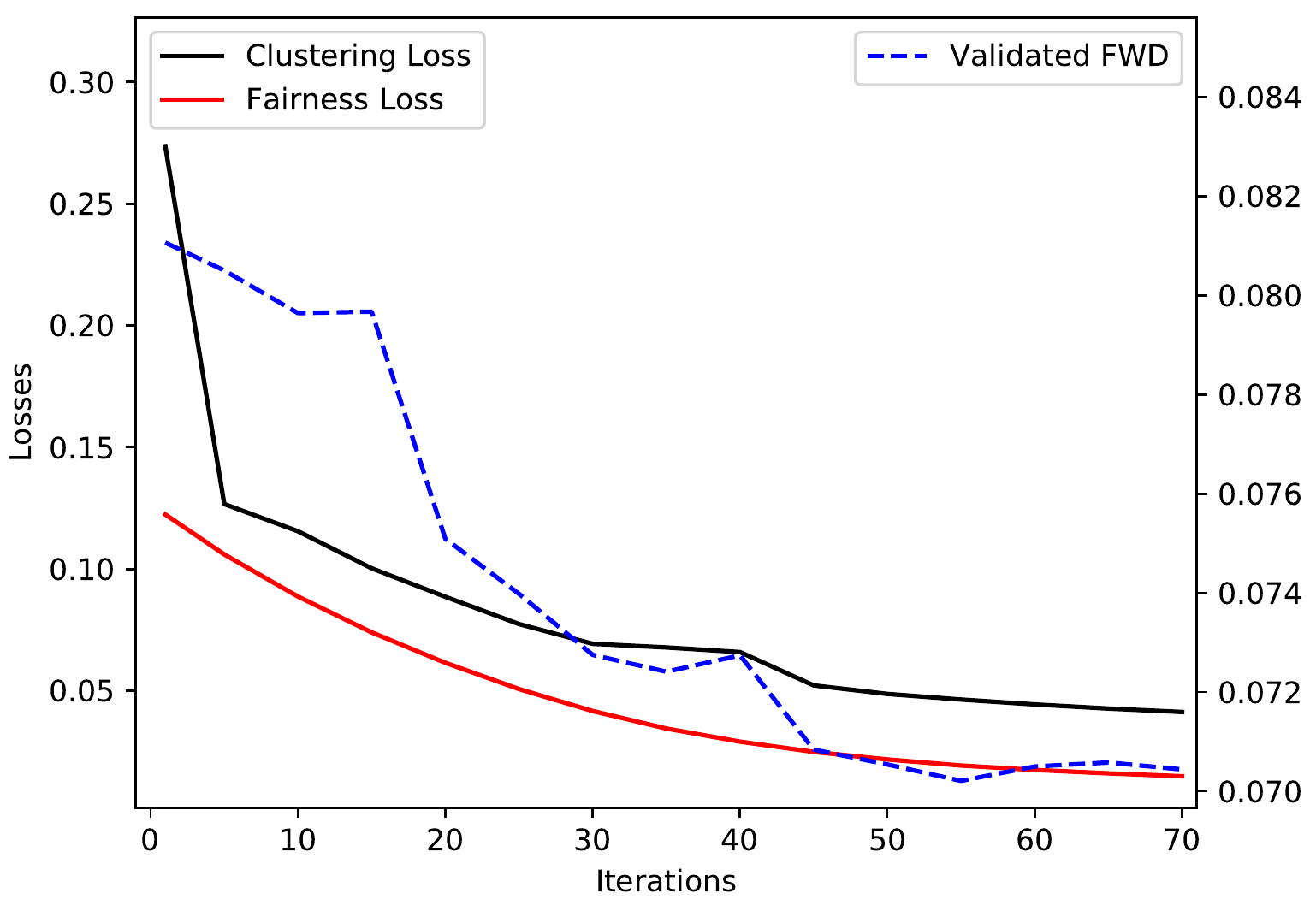}
		\caption{Visualized learning curves of clustering loss, fairness loss and FWD on validation set on HAR dataset.}
		\label{fig:curve}
		\setlength{\belowcaptionskip}{-0.1cm}
	\end{figure}
	
	{\setlength\abovedisplayskip{5pt}
		\begin{table*}[h]
			\begin{center}
				\setlength{\tabcolsep}{2mm}{
					\begin{tabular}{c|ccc|ccc|ccc}
						\toprule[2pt]
						& \multicolumn{3}{c|}{HAR} &  \multicolumn{3}{c|}{Adult} & \multicolumn{3}{c}{D\&S}\\
						\hline
						Methods & ACC $\uparrow$ & NMI  $\uparrow$ & FWD $\downarrow$ & ACC $\uparrow$& FWD $\downarrow$ & Balance $\uparrow$& ACC $\uparrow$ & NMI $\uparrow$& FWD $\downarrow$\\ \hline
						k-means++ &{$0.597$}  &$0.592 $& $0.109$&$0.714$&$0.417$& $0.312$& $0.647$&$0.732$&{$0.508$}\\ 
						DBSCAN & $0.501$ & $0.401$ &$ 0.773$ &$0.753$&$1.000$&$0.000$&$0.556$&$0.535$&$1.232$\\ 
						SEC & $0.481$ & $0.307$ & $0.093$ &$0.743$&$0.940$&$0.027$&$0.524$&$0.507$&$0.546$\\ 
						\hline AE + KM &$0.558 $ &$0.613 $&$0.092 $&$\mathbf{0.758}$&$0.438$&$0.299$&$0.766$&$0.782$&$0.537$\\ 
						DCN &$0.577 $ &$0.612 $& $0.102 $&$\mathbf{0.758}$&$0.441$&$0.295$&$0.678$&$0.751$&$0.681$\\ 
						DEC &$0.571 $&$\mathbf{0.662} $&$0.097$&{$0.744$}&{$0.396$}&{$0.363$}&$\mathbf{0.796}$&$\mathbf{0.825}$&$0.554$\\
						\hline Ours &$\mathbf{0.607} $&{$0.661 $}&$\mathbf{0.089}$&$0.735$&$\mathbf{0.368}$& $\mathbf{0.430}$&{$0.788$}&{$0.818$}&$\mathbf{0.462}$\\	
						\bottomrule[2pt]
					\end{tabular}
				}
			\end{center}
			\caption{Experimental results (over 10 trials) on HAR, Adult dataset, and D\&S data sets. The ``Balance'' metric is defined in \cite{chierichetti2017fair}. Best performance is in \textbf{bold black}. ``$\uparrow$'' and ``$\downarrow$'' mean that larger value indicates better performance or vice versa.}
			\label{table:quant}
		\end{table*}
		\setlength\belowdisplayskip{5pt}}
	\begin{figure*}[h]
		\centering
		\subfloat[Initial histogram]{\includegraphics[width=0.33\linewidth]{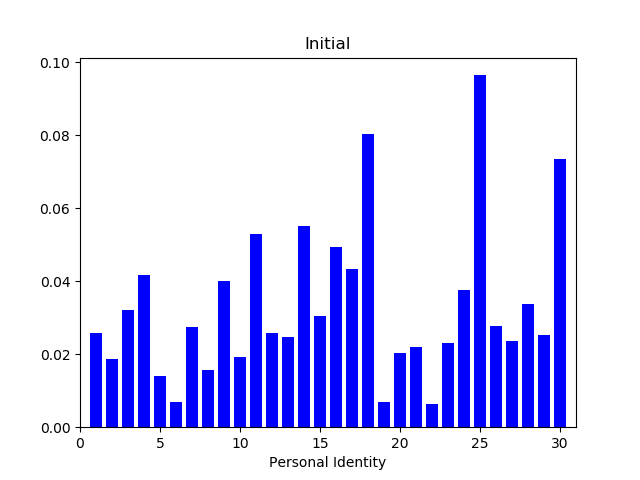}%
			\label{fig:init}}
		\subfloat[Histogram after 3 Iterations]{\includegraphics[width=0.33\linewidth]{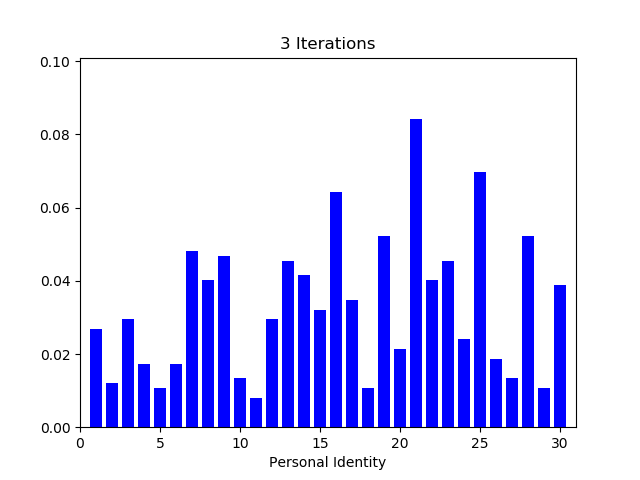}%
			\label{fig:3_iter}}
		\subfloat[Histogram after 5 Iterations]{\includegraphics[width=0.33\linewidth]{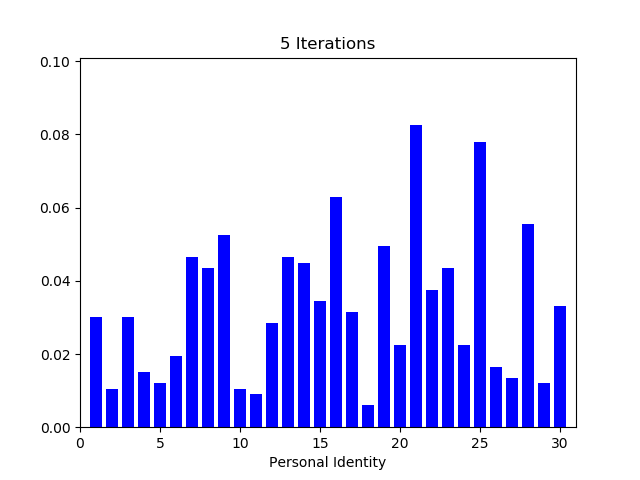}%
			\label{fig:5_iter}}
		%\subfloat[]{\includegraphics[width=0.25\linewidth]{./pics/7_iter_fair.png}%
		%\label{fig:7_iter}}
		\caption{Histograms of the unfairest cluster (with largest FWD) across the training stages on HAR dataset.}
		\label{fig: fair_iter}
	\end{figure*}
	\subsubsection{Answ1: Fairness of existing approaches.} 
	
	We evaluate the fairness of 3 deep clustering approaches and 3 non-deep clustering approaches in terms of FWD on HAR, Adult and D\&S datasets (Table \ref{table:quant}). It can be seen that there is no necessary relation between clustering quality and fairness. Besides, deep clustering with dimensionality reduction is not necessarily fairer than shallow methods which perform clustering on original features, although deep clustering methods tend to achieve higher clustering accuracy. Among shallow clustering approaches, density-based method DBSCAN seems to be the unfairest one because the number of clusters in DBSCAN can not be assigned. If the algorithm automatically chooses large number of clusters, the imparity within single cluster could be aggravated.
	
	\subsubsection{Answ2: Performance of our model} 
	
	Performance of our model on HAR, Adult and D\&S datasets, in terms of both traditional clustering accuracy, normalized mutual information and fairness, are shown in the last row in Table. \ref{table:quant}. As for clustering quality measured by ACC and NMI, performance of our model is mildly lower or ties with the base model DEC. On the other hand, the fairness measured by FWD is remarkably improved ($10.3\%$ relative improvement on HAR, $7\%$ on Adult and $16.6\%$ on D\&S). 
	\begin{figure*}[t]
		\centering
		\subfloat[DEC]{\includegraphics[width=0.38\linewidth]{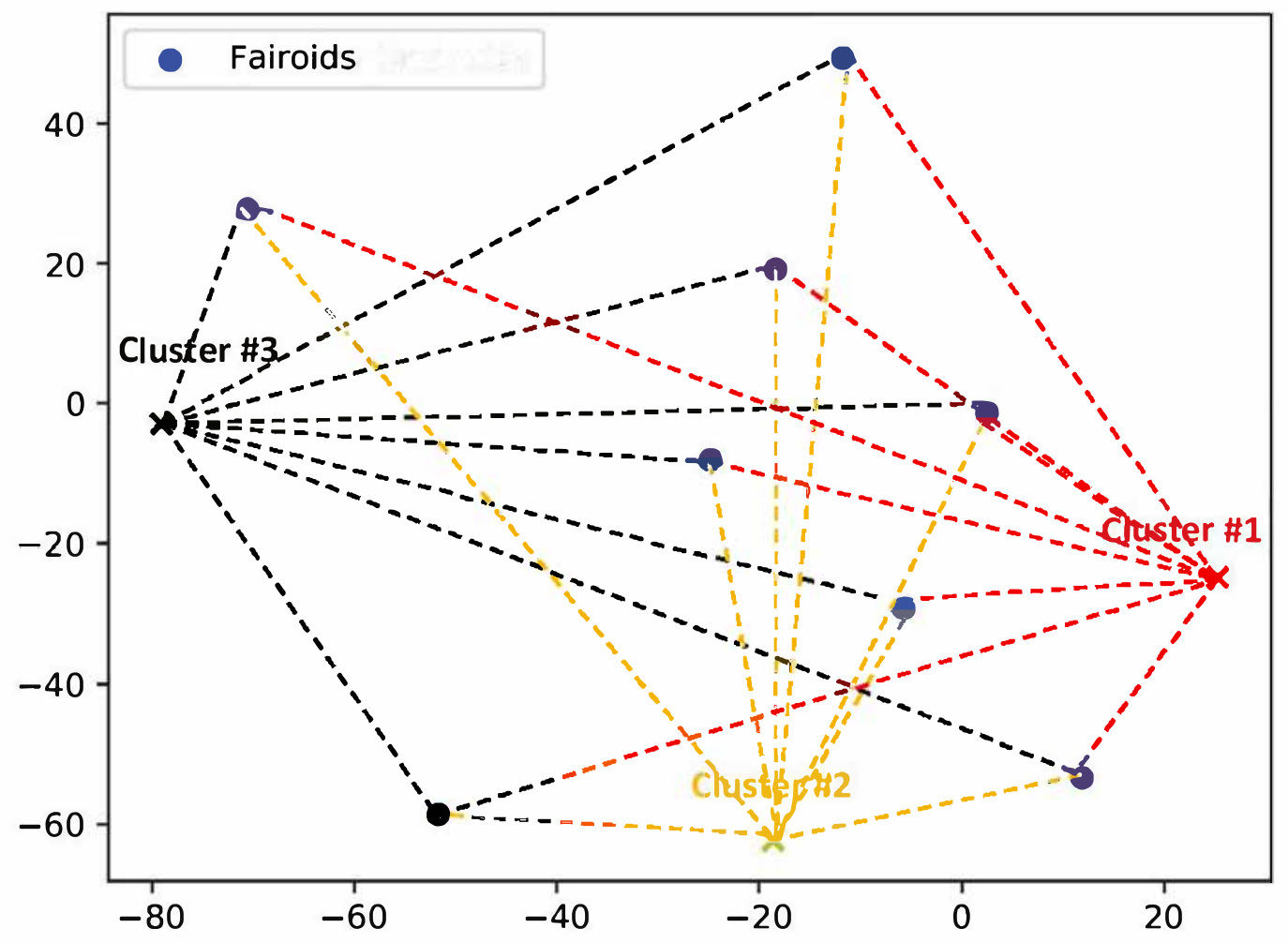}%
			\label{fig: tsne_init}}
		\subfloat[Our approach]{\includegraphics[width=0.38\linewidth]{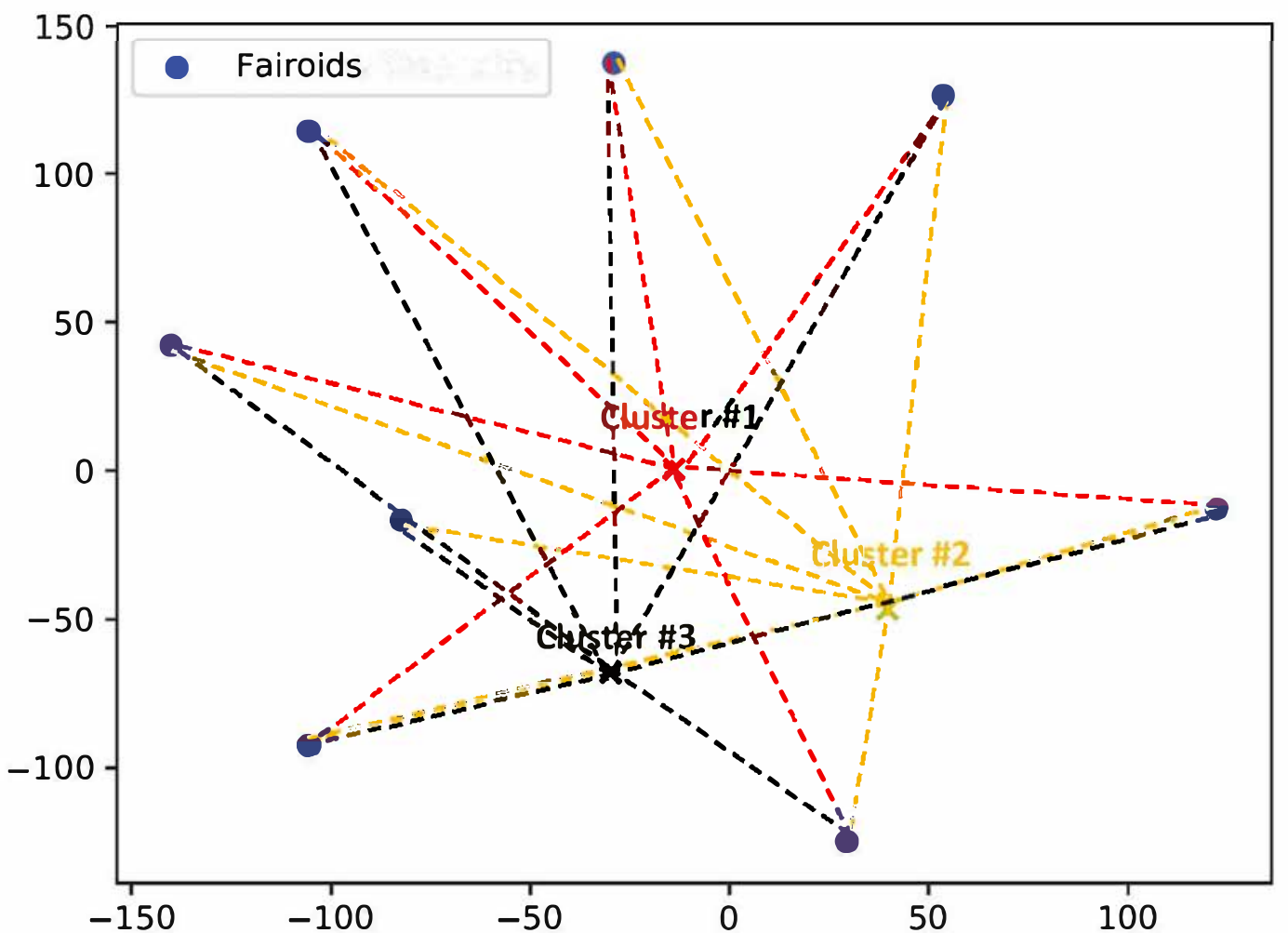}%
			\label{fig: tsne_ours}}
		\caption{t-SNE embeddings of resulting centroids and fairoids from deep embedded clustering (DEC) and our approach.}
		\label{fig: tsne}
	\end{figure*}
	\begin{figure}[h]
		\centering
		\includegraphics[width=0.8\linewidth]{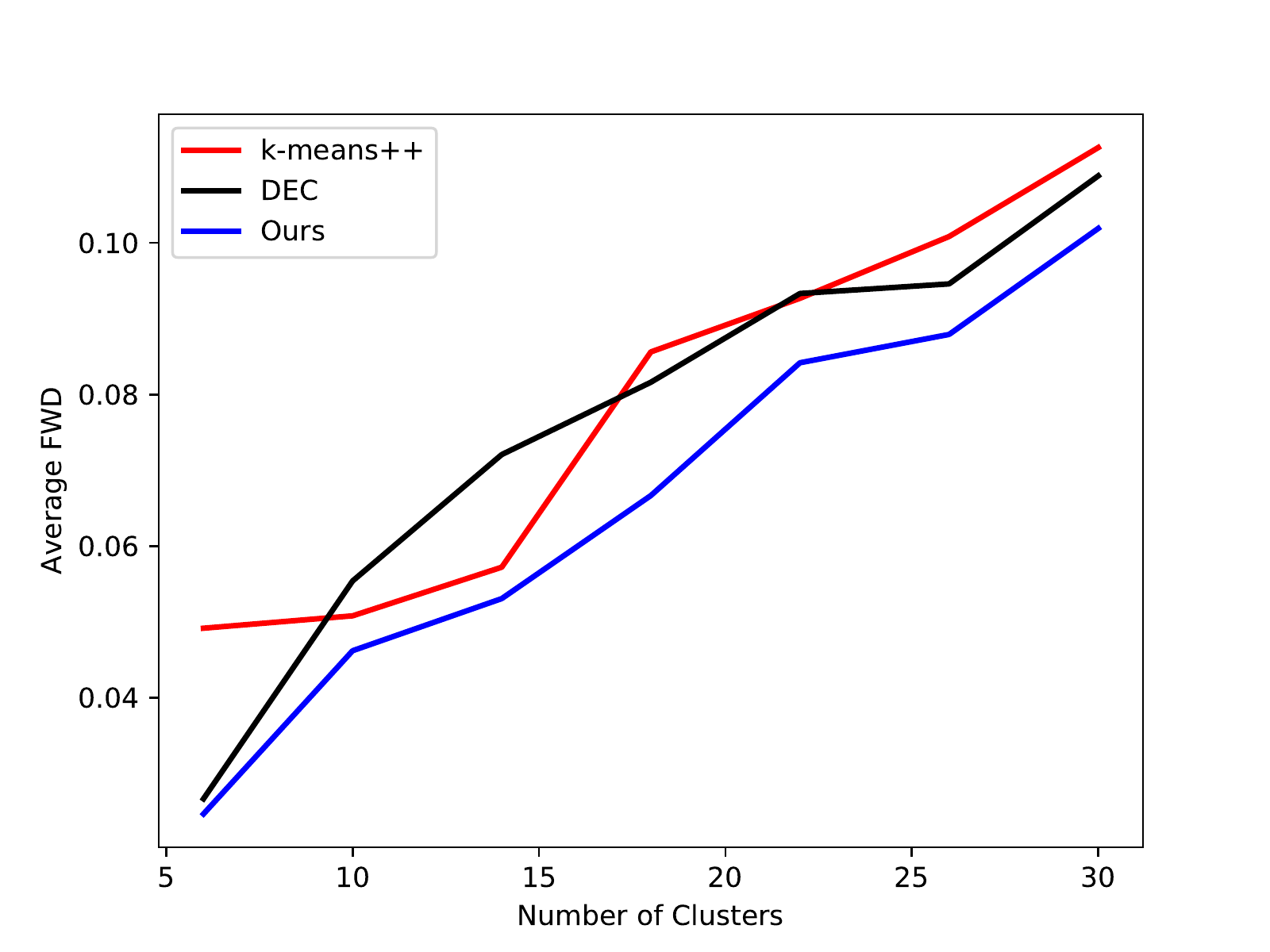}
		\caption{FWD per cluster with increasing number of clusters.}
		\label{fig: average_fwd}
	\end{figure}
	We visualize the histograms of the unfairest cluster at the initial stage and after 3 and 5 iterations (see Figure \ref{fig: fair_iter}). It is shown that the balance of this unfairest cluster is largely improved. Besides, we also present the learning curves of clustering and fairness losses and the measured fairness on HAR dataset (shown in Figure \ref{fig:curve}), which shows the fairness loss has no problem to be simultaneously optimized with clustering objective and the fairness is effectively improved. To further verify how our fairness objective works, we apply DEC and our approach on downsampled data from three clusters in D\&S dataset. The t-SNE embeddings of resulting cluster and fairoids from DEC \cite{DBLP:conf/icml/XieGF16} and our approach are visualized in Figure \ref{fig: tsne}, which confirms that distances from the cluster centroids to embedding of protected groups are effectively balanced but the cluster centroids seem to be less separated. This may explain why our fairness objective can improve fairness of resulting clusters but lead to loss in accuracy as shown in Table \ref{table:quant}. 
	
	Moreover, we testify if our approach can steadily learn fairer clusters when number of clusters $K$ is increasing. Intuitively, monochromatic cluster is more likely to appear if there are more clusters. As shown in Figure \ref{fig: average_fwd}, it is not surprising that average FWD per cluster is increasing with more clusters in each of $k$-means, DEC and our approach. However, our approach consistently learns fairer clusters.
	
	\subsubsection{Answ3: Fairness measured by balance when $T=2$.} 
	
	To verify that if the fairness measured based on our definition has similar trend to definition and measure in \cite{chierichetti2017fair} when $T = 2$, we also utilize the balance score to evaluate the fairness of our clustering result on Adult data set. As seen in the last column in Table. \ref{table:quant}, our model dramatically improves the balance of the unfairest cluster by $18\%$ compared to DEC. Moreover, the fairness measured by balance score is verified to has the same trend as FWD for all clustering approaches, which has been theoretically pointed out in Proposition \ref{cv-score}. 
	\subsubsection{Answ4: Trade-off between Clustering Accuracy and Fairness} \label{tradeoff_disc}
	Because the fairness loss could be seen as a ``fairness'' regularization term for clustering loss, we are interested in how the hyper-parameter $\gamma$ controls the trade-off between clustering accuracy and fairness. Thus we testify our model under different values of $\gamma$ ranging from $10^{-2}$ to $10^3$ on Adult dataset, the results can be seen in Figure \ref{fig:tradeoff}. With larger $\gamma$ our model is fairer but less accurate. 
	{
		\begin{figure}[h]
			\centering
			\includegraphics[width=0.75\linewidth]{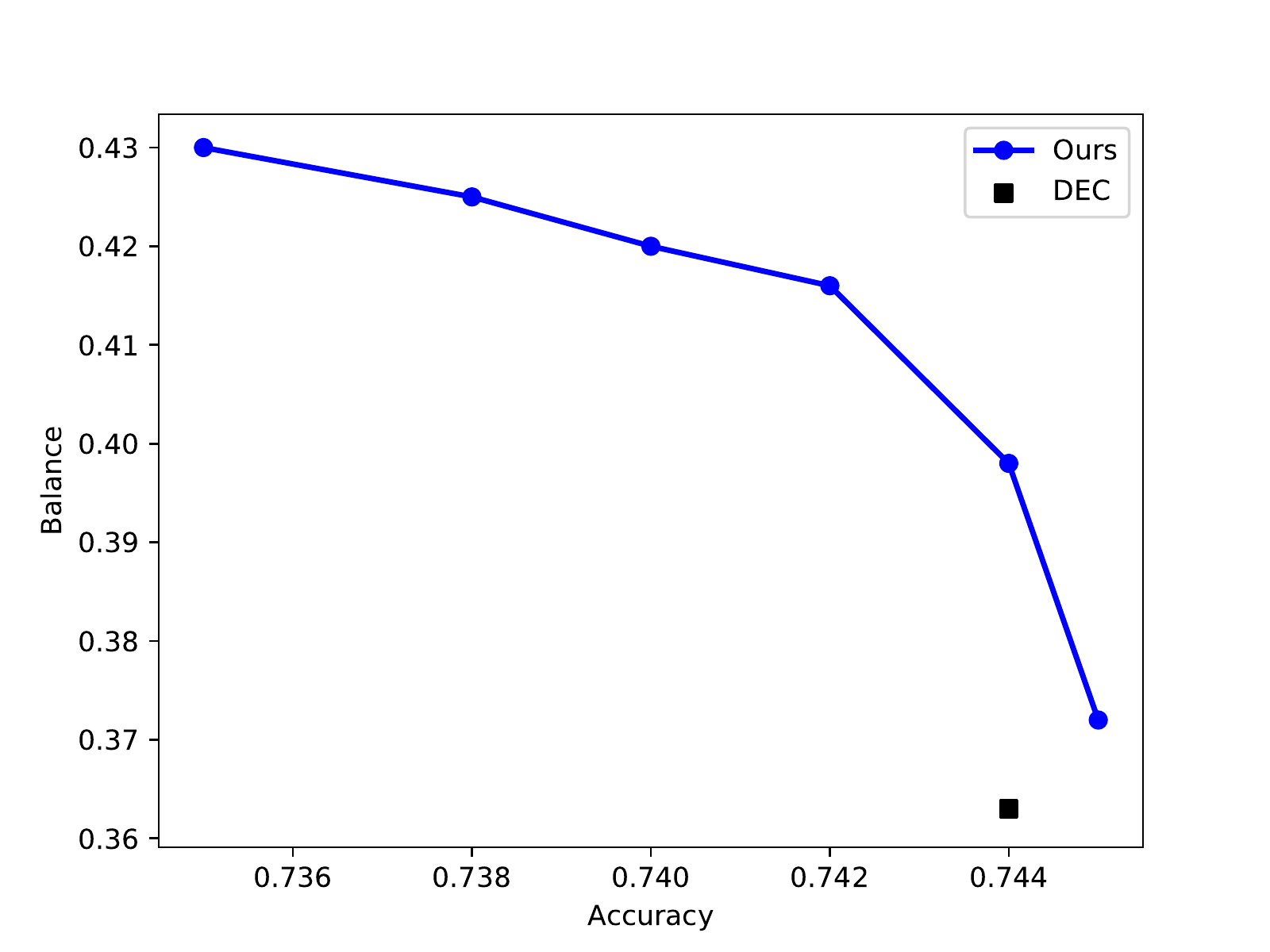}
			\caption{Trade-off between the clustering accuracy and balance of clusters on Adult dataset when hyper-parameter $\gamma$ varies. The black point shows clustering accuracy and FWD of base model DEC and the blue line shows the performance of ours.}
			\label{fig:tradeoff}
		\end{figure}
		\setlength\belowdisplayskip{1pt}}
	\section{Conclusion}
	In this work, we define the fair clustering problem with multi-state protected variable under group fairness and disparate impact doctrines. Based on the definition, we propose a scalable framework to improve fairness of deep clustering. We compare our approach with six representative clustering approaches on three publicly available data sets with different types of protected attributes. The experimental results demonstrate our approach can steadily improve fairness. By controlling one hyper-parameter, our approach can provide flexible trade-off between clustering accuracy and fairness. 
	%Remaining problems include investigating fair clustering problem under another important notion of fairness: disparate treatment and measuring the algorithmic hardness of fair clustering under different numbers of protected groups and clusters. 

	
	\newpage
	\newpage
	\bibliography{example_paper}
	\bibliographystyle{icml2019}

	
	%%%%%%%%%%%%%%%%%%%%%%%%%%%%%%%%%%%%%%%%%%%%%%%%%%%%%%%%%%%%%%%%%%%%%%%%%%%%%%%
	%%%%%%%%%%%%%%%%%%%%%%%%%%%%%%%%%%%%%%%%%%%%%%%%%%%%%%%%%%%%%%%%%%%%%%%%%%%%%%%
\end{document}